\documentclass{article} 
\usepackage{iclr2027_conference,times}

\usepackage{amsmath,amsfonts,bm}

\def\eqref#1{equation~\ref{#1}}

\def\1{\bm{1}}

\DeclareMathAlphabet{\mathsfit}{\encodingdefault}{\sfdefault}{m}{sl}
\SetMathAlphabet{\mathsfit}{bold}{\encodingdefault}{\sfdefault}{bx}{n}

\usepackage{hyperref}
\usepackage{url}
\usepackage{amsmath,amssymb,amsfonts}
\usepackage{amsthm}
\newtheorem{proposition}{Proposition}
\usepackage{textcomp}
\usepackage{xcolor}

\usepackage[utf8]{inputenc} 
\usepackage[T1]{fontenc}    
\usepackage{nicefrac}       
\usepackage{mathtools}    
\usepackage{bm}      
\usepackage{float}
\newif\ifworkinprogress
\workinprogresstrue 

\ifworkinprogress

\else

\fi
\usepackage{enumitem}
\usepackage{graphicx}
\usepackage{wrapfig}           
\usepackage{multirow}
\usepackage{booktabs}
\usepackage{algorithm}
\usepackage{algpseudocode}

\usepackage{caption}
\usepackage{subcaption}

\title{ChronoSheaf: Joint Spatio-Temporal Spectral Filtering on Cellular Sheaves}

\author{Antiquus S.~Hippocampus, Natalia Cerebro \& Amelie P. Amygdale \thanks{ Use footnote for providing further information
about author (webpage, alternative address)---\emph{not} for acknowledging
funding agencies.  Funding acknowledgements go at the end of the paper.} \\
Department of Computer Science\\
Cranberry-Lemon University\\
Pittsburgh, PA 15213, USA \\
\texttt{\{hippo,brain,jen\}@cs.cranberry-lemon.edu} \\
\And
Ji Q. Ren \& Yevgeny LeNet \\
Department of Computational Neuroscience \\
University of the Witwatersrand \\
Joburg, South Africa \\
\texttt{\{robot,net\}@wits.ac.za} \\
\AND
Coauthor \\
Affiliation \\
Address \\
\texttt{email}
}

\author{Abeer Mostafa$^{1}$, Zahra Ahmadi$^{1,2}$\\
$^1$Peter L. Reichertz Institute for Medical Informatics, Hannover Medical School, Germany\\
$^2$Lower Saxony Center for Artificial Intelligence and Causal Methods in Medicine (CAIMed), \\Hannover, Germany\\
\texttt{\{Mansour.Abeer, Ahmadi.Zahra\}@mh-hannover.de}
}

\iclrfinalcopy 
\begin{document}

\maketitle

\begin{abstract}

Sheaf neural networks generalize graph neural networks by assigning each edge a learned linear restriction map that specifies how the local representations of adjacent nodes relate to one another, thereby mitigating oversmoothing. Existing sheaf architectures, however, define the sheaf Laplacian as an instantaneous spatial operator and therefore cannot explicitly model how signals evolve over time on a graph. We first introduce a spectral sheaf Laplacian, defined jointly over space and time, such that a single operator regulates signal propagation across both the spatial graph and the temporal domain. Directly extending this operator's receptive field in space or time would require increasingly complex sheaf structures. Instead, we propose a spectral filter that expands the effective receptive field along both dimensions without enlarging the underlying sheaf.
Building on this operator, we introduce \textbf{ChronoSheaf}, a spatio-temporal sheaf network with two key components. First, Transitional Restriction Maps (TRM) decompose each node's restriction map into a spatial term and two temporal frames that independently control its relation to neighboring nodes at the preceding and following timestep, enabling asymmetric temporal influence. Second, the Dual Spectral Filter (DSF) applies a single bounded-degree polynomial filter with an adaptively estimated spectral radius to the joint spatio-temporal sheaf operator. We validate ChronoSheaf on a diverse set of spatio-temporal forecasting benchmarks, where it outperforms state-of-the-art baselines. The code is available at: \url{https://anonymous.4open.science/r/ChronoSheaf-1D7D/}.
\end{abstract}

\section{Introduction}
\label{sec:introduction}

Sheaf neural networks generalize graph neural networks by assigning a vector space (stalk) to each node and edge and learning linear restriction maps between them, which lets them model heterogeneous relationships. The ordinary graph Laplacian is recovered as a special case of the constant sheaf, where every stalk is identical, and every restriction map is the identity. By letting restriction maps be learned and non-trivial, \citet{bodnar2022neural} showed that a node's local frame need not agree with its neighbors' outright, giving sheaf neural networks a proven advantage over conventional graph neural networks on heterophilic graphs and against oversmoothing. In every such construction, however, the sheaf is defined purely over the spatial graph. Even when restriction maps are learned from node features \citep{bodnar2022neural}, the base space contains no notion of time, so the sheaf Laplacian only learns from nodes observed at the same instant, and the associated diffusion relaxes toward a static equilibrium, the space of harmonic cochains \citep{hansen2019spectral}. As a consequence, increasing depth or filter order does not let a spatial sheaf Laplacian reach across time. Hence, temporal context must come from a separate module outside the sheaf that does not benefit from its learned local frames.

This restriction has a direct architectural consequence for spatio-temporal graph learning. The dominant paradigm thereby depends on conventional GNN architectures which are prone to oversmoothing and compose two separate operators, a spatial graph operator, and a recurrent, convolutional, or attention-based module for temporal dependence
\citep{li2018diffusion,yu2018spatio,wu2019graph,bai2020adaptive,zheng2020gman,wu2020connecting}. \citet{marisca2025stoversquashing} recently showed this composition has a measurable cost; because causal temporal convolutions propagate information only forward, deeper stacks accumulate disproportionately more paths from early timesteps than from recent ones, so convolutional spatio-temporal GNNs counterintuitively favor temporally distant information over nearby steps.

We address this gap by defining a single spectral sheaf operator jointly over space and time, so that one filtering recursion controls both spatial and temporal dependence. This operator extends the sheaf domain from the spatial graph alone to a space-time graph, and edges connect spatial neighbors at a shared time as well as each node to its neighbors at the adjacent timestep. The network architecture built around this operator is ChronoSheaf, consisting of two components. Transitional restriction maps decompose each node's restriction map into a learned spatial frame that determines how spatial neighbors are compared, plus two further frames, $F^{\mathrm{in}}$ and $F^{\mathrm{out}}$, that independently weight a node's interaction with its neighbors at the preceding and following timestep.
The dual spectral filter then applies a single spectral filter to this joint operator, giving ChronoSheaf a receptive field that grows jointly in graph hops and temporal lag from one spectral recursion (Figure~\ref{fig:intro}). 
Our contributions are as follows:
\begin{itemize}[leftmargin=20pt]
    \item We introduce a new class of sheaf GNNs, called ChronoSheaf, in which we fundamentally rethink the three core components of sheaf GNNs: 
    \begin{enumerate}
        \item The \textit{base space} of the sheaf is a joint space-time graph, so that each node's stalk is compared not only with its spatial neighbors but also with its own state and its neighbors' states at adjacent timesteps.
        \item The \textit{restriction maps} allow each node to respond to its neighbors' past and future states with distinct strengths to capture asymmetric temporal influence.
        \item The \textit{sheaf Laplacian}'s spatial hops and temporal lags are controlled by a spectral filter that acts on the Laplacian's frequency spectrum without needing eigendecomposition. We further provide a theoretical analysis of the resulting operator, characterizing its spectral properties and its joint spatio-temporal receptive field.
    \end{enumerate}
      
    \item We conduct extensive experiments on six real-world datasets spanning diverse domains. The results show that ChronoSheaf resists oversmoothing, captures long-range dependencies more effectively, and outperforms state-of-the-art methods by a notable margin.

\end{itemize}

\begin{figure}[!ht]
    \centering
    \includegraphics[width=0.92\linewidth]{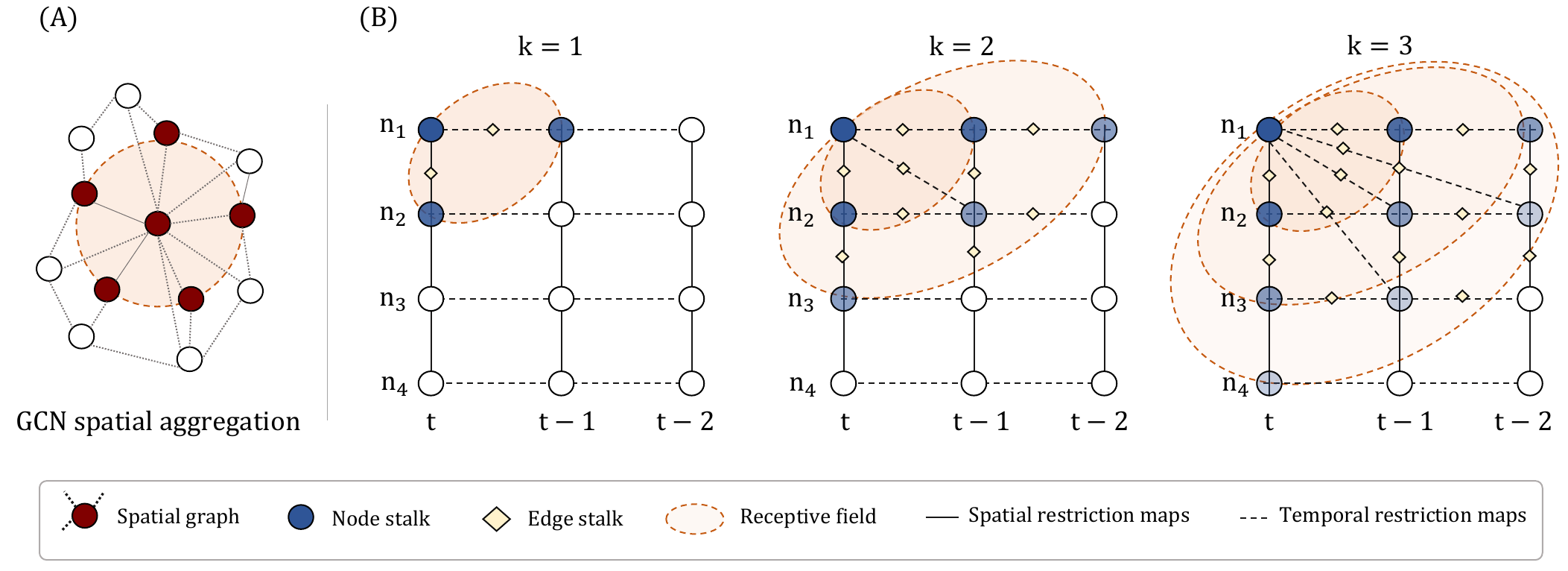}
    \caption{(A) State-of-the-art graph forecasting models aggregate spatial information from neighboring nodes and enlarge their receptive field by stacking additional layers. (B) Our model performs spectral filtering directly on joint space-time cellular sheaves. Higher-order polynomial filters extend the receptive field to larger neighborhoods and capture deeper temporal context.}
    \label{fig:intro}
\end{figure}

\section{Related Work}
\label{sec:related_work}

\textbf{Spectral graph convolution and its limits.}
Spectral filtering on graphs originates in \cite{bruna2014spectral} and is made scalable via polynomial truncation in \cite{defferrard2016convolutional}, whose first-order simplification underlies the method proposed in \cite{kipf2017gcn}. \citet{velickovic2018gat} replace the fixed spectral filter with data-dependent attention weights. Known failure modes of this design, oversmoothing \citep{rusch2023survey, li2018oversmoothing} and oversquashing \citep{alon2021bottleneck}, motivate architectures that enrich what a node can represent. This is the main gap targeted by sheaf neural architectures.

\textbf{Sheaf neural networks.}
\cite{hansen2019spectral} generalize the graph Laplacian to a Hodge Laplacian over a cellular sheaf using linear restriction maps defined between per-node stalks. \cite{bodnar2022neural} instantiate this as a neural network, with a fully parametric diffusion architecture and provable control over Dirichlet-energy decay, showing a learned non-trivial sheaf mitigates oversmoothing where the trivial sheaf cannot. Follow-up work restricts the space of restriction maps: with orthogonal Laplacians \citep{barbero2022connection} and $O(d)$-bundle maps with continuous transport \citep{bamberger2025bundle}. Most recently, \cite{hajij2025copresheaf} subsume sheaf diffusion, GNNs, and attention into a single copresheaf-on-combinatorial-complex formalism. \cite{ribeiro2026cooperative} define cellular sheaves over directed graphs, allowing the diffusion operator to selectively listen to distant nodes and mitigate oversquashing. \cite{fiorini2026sheaves} also propose a directed cellular sheaf and its associated directed sheaf Laplacian, which encodes both graph topology and edge orientation. These extensions, however, address spatial connectivity specifically. From a different angle, we address two questions: how time can be encoded within the sheaf and how to learn the sheaf Laplacian's spectrum.

\textbf{Spatio-temporal graph learning.}
The dominant paradigm combines a spatial graph operator with a separate temporal mechanism, including recurrent units \citep{li2018diffusion}, gated temporal convolutions \citep{yu2018spatio,wu2019graph}, and self-attention \citep{zheng2020gman}. Other approaches enrich the spatial or dynamical component through adaptive graph learning \citep{bai2020adaptive}, zigzag-persistence representations \citep{chen2021z}, neural ODEs \citep{fang2021stgode}, or controlled differential equations \citep{choi2022graph}.
Recent work continues to push the adaptive graph side toward efficiency: \cite{ma2025mage} replace quadratic-cost adaptive graph learning with a linear, mixture-of-experts formulation. Across these methods, spatial and temporal dependence are usually modeled by two distinct operators, composed architecturally, so receptive fields grow additively with depth. A few works operate on a joint space-time domain; STSGCN \citep{song2020spatial} connects each node to itself and its neighbors at adjacent timesteps within a localized space-time graph and graph-time convolutional networks \citep{sabbaqi2023graph} filter over parametric product graphs. These methods demonstrate the value of a joint domain, but they weight its edges with scalars: every node shares a single feature space with its past, its future, and its neighbors, which is precisely the reason behind oversmoothing.

\section{ChronoSheaf}
\label{sec:method}

This section describes ChronoSheaf. Section~\ref{sec:prelim} establishes the sheaf-theoretic and spectral preliminaries. Section~\ref{sec:trm} introduces the transitional restriction maps, the learned frames that control the space-time domain. Section~\ref{sec:laplacian} assembles these frames into the three-part sheaf Laplacian and proves it is a real, symmetric, positive semi-definite operator. Section~\ref{sec:dsf} introduces the dual spectral filter, with the adaptive spectral-radius bound. Finally, Section~\ref{sec:layer-network} describes the proposed layer and the full network.

\subsection{Preliminaries}
\label{sec:prelim}

\paragraph{Spatial graph.} Let $\mathcal{G} = (\mathcal{V}, \mathcal{E})$ be a connected, undirected graph with $|\mathcal{V}| = N$ nodes, and adjacency matrix $A \in \{0,1\}^{N\times N}$. We work with the symmetrically normalized adjacency $\tilde{A} = D^{-1/2}AD^{-1/2}$, where $D = \mathrm{diag}(A\mathbf{1})$, so that the normalized graph Laplacian is $L_{\mathrm{sym}} = I - \tilde{A}$.

\paragraph{Cellular sheaves.} A cellular sheaf $\mathcal{F}$ on $\mathcal{G}$ assigns to each node $v$ a stalk $\mathcal{F}(v) \cong \mathbb{R}^r$ and, for each incident node-edge pair $v \triangleleft e$, a linear restriction map $\mathcal{F}_{v\triangleleft e}: \mathcal{F}(v) \to \mathcal{F}(e)$ into a shared edge stalk \citep{hansen2019spectral}. The associated sheaf Laplacian
\begin{equation}
(L_{\mathcal{F}} x)_v = \sum_{e=(u,v)\in\mathcal{E}} \mathcal{F}_{v\trianglelefteq e}^\top\big(\mathcal{F}_{v\trianglelefteq e}x_v - \mathcal{F}_{u\trianglelefteq e}x_u\big),
\label{eq:sheaf_L}
\end{equation}
is a symmetric, positive semi-definite operator whose quadratic form $x^\top L_{\mathcal{F}} x = \sum_e \|\mathcal{F}_{v\triangleleft e}x_v - \mathcal{F}_{u\triangleleft e}x_u\|_2^2$ measures disagreement between adjacent nodes after each node's local frame is applied \citep{bodnar2022neural}. Setting every restriction map to the identity on $\mathbb{R}^1$ recovers the ordinary, unnormalized combinatorial graph Laplacian $D-A$. The sheaf with identity restriction maps is called the trivial sheaf.

\paragraph{Spectral graph theory.} For a symmetric operator $L$ with eigendecomposition $L = U\Lambda U^\top$, a spectral graph filter applies a function $g(\Lambda)$ of the eigenvalues before returning to the node domain, $g(L)x = U g(\Lambda) U^\top x$. 

In this work, we do not compute this eigendecomposition explicitly, because of the computationally expensive overhead it adds to the network. We use its polynomial (Chebyshev) approximation \citep{defferrard2016convolutional}: with $\tilde{L}$ rescaled to $[-1,1]$, the order-$K$ Chebyshev filter is
\begin{equation}
g_\theta(L)x = \sum_{k=0}^{K}\theta_k\, T_k(\tilde{L})\,x, \qquad T_0(z)=1,\ T_1(z)=z,\ T_k(z) = 2zT_{k-1}(z) - T_{k-2}(z),
\end{equation}
\begin{wrapfigure}{r}{0.44\textwidth}
\vspace{-15pt}
    \centering
    \includegraphics[width=\linewidth]{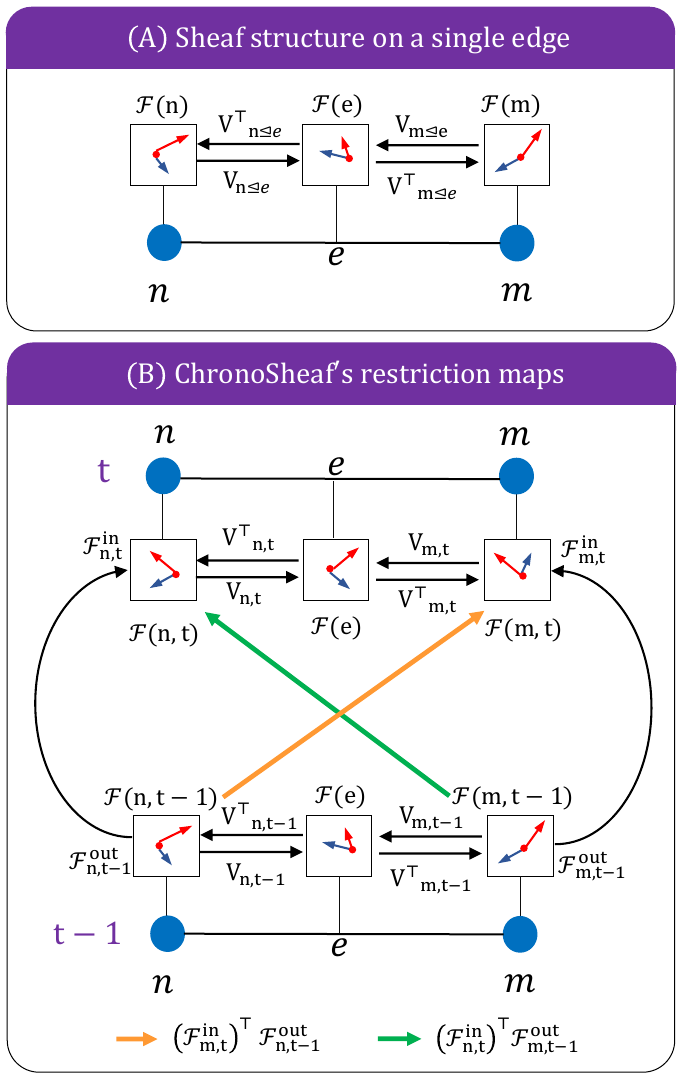}
    \caption{\small (A) Neural sheaf structure~\cite{bodnar2022neural}. (B) ChronoSheaf's joint space-time sheaf structure. }
    \label{fig:sheaf}
    \vspace{-50pt}
\end{wrapfigure}
which is exactly $K$-localized (depends on nodes within $K$ hops) and requires no eigendecomposition.

\paragraph{Joint space-time sheaf domain.} ChronoSheaf's central construction is a cellular sheaf defined not over $\mathcal{G}$ alone, but over a joint space-time graph $\mathcal{G}^\star = (\mathcal{V}^\star,\mathcal{E}^\star)$. Nodes are pairs $(n,t)$ for $n\in\mathcal{V}$ and $t\in\{1,\dots,P\}$, where $P$ is the number of observed timesteps. $\mathcal{E}^\star$ contains a spatial edge between $(n,t)$ and $(m,t)$ whenever $(n,m)\in\mathcal{E}$, and a temporal edge between $(n,t)$ and $(m,t-1)$ for every $m$ that is either $n$ itself (a self-loop) or a spatial neighbor of $n$. Sections~\ref{sec:trm}-~\ref{sec:laplacian} add to $\mathcal{G}^\star$ node-homogeneous frames and assemble its sheaf Laplacian; Section~\ref{sec:dsf} shows that a single Chebyshev recursion over this one operator is what lets ChronoSheaf's receptive field grow jointly in graph hops and temporal lag.

\paragraph{Forecasting Task.} The model takes input signal $x \in \mathbb{R}^{N \times P \times d_{\text{in}}}$, obtained from the raw observations after standard preprocessing (segmentation into $P$ timesteps of $d_{\text{in}}$ channels). Given a history of such signals, the goal is to learn a mapping $F$ that predicts the next $H$ steps of the underlying scalar measurement at each node.

\subsection{Transitional Restriction Maps (TRM)}
\label{sec:trm}

ChronoSheaf's first component assigns learned frames to the joint graph $\mathcal{G}^\star$. Let $B \in \mathbb{R}^{3 \times N \times r \times d}$ be a learnable base tensor, holding one candidate frame of each of three kinds (spatial, outgoing-temporal, incoming-temporal frames) per node. At each timestep, a learned gate rescales the three candidates for node $n$ as a function of the layer input $h_{n,t} \in \mathbb{R}^{d}$:
\begin{equation}
g_{n,t} = \mathbf{1} + \tfrac{1}{2}\tanh(W_g h_{n,t}), \qquad W_g \in \mathbb{R}^{3 \times d},
\end{equation}
\begin{equation}
    V_{n,t} = g_{n,t,1} B_{1,n}, \qquad F^{\mathrm{out}}_{n,t} = g_{n,t,2} B_{2,n}, \qquad F^{\mathrm{in}}_{n,t} = g_{n,t,3} B_{3,n},
\end{equation}

so that $V_{n,t}, F^{\mathrm{in}}_{n,t}, F^{\mathrm{out}}_{n,t} \in \mathbb{R}^{r \times d}$ are node $n$'s three restriction maps at time $t$. Since $\tanh(\cdot)\in(-1,1)$, each component of $g_{n,t}\in\mathbb{R}^3$ lies in $(0.5,1.5)$ to bound every frame's scale for numerical stability.

We call $F^{\mathrm{in}}$ and $F^{\mathrm{out}}$ \emph{transitional restriction maps}. $V_{n,t}$ plays the role of an ordinary spatial sheaf frame, learning how node $n$'s representation is compared with its spatial neighbors at the same time. $F^{\mathrm{out}}_{n,t}$ and $F^{\mathrm{in}}_{n,t}$ learn how $n$'s representation is compared across a single-step transition in time: $F^{\mathrm{out}}_{n,t}$ sets how strongly $n$ projects its current signal onto the following timestep, and $F^{\mathrm{in}}_{n,t}$ sets how strongly $n$ accepts a neighbor's projection from the preceding one. Because these are two independently learned maps, a node can weight its sensitivity to the past differently from its sensitivity to the future. 

\subsection{The ChronoSheaf Laplacian}
\label{sec:laplacian}

\paragraph{Spatial term.} We instantiate the sheaf Laplacian of Eq.~\ref{eq:sheaf_L} to node-homogeneous restriction maps, in which every spatial edge incident to $(n,t)$ uses the same spatial frame, $\mathcal{F}_{(n,t) \trianglelefteq e} = V_{n,t}$, and weight each edge by the normalized adjacency $\tilde{A}$:
\begin{equation}
    (L_{sp}x)_{n,t} = \sum_{m} \tilde{A}(n,m) V_{n,t}^\top \left( V_{n,t} x_{n,t} - V_{m,t} x_{m,t} \right). \label{eq:lsp}
\end{equation}
Because $V_{n,t} \in \mathbb{R}^{r \times d}$, the restriction maps project the $d$-dimensional feature stalks into a shared $r$-dimensional edge space before measuring disagreement.

\paragraph{Temporal terms.} Let $\hat{A} = D^{-1/2}(A + I)D^{-1/2}$ be the self-loop-augmented normalized adjacency, and let $\hat{d}(n) := \sum_m \hat{A}(n, m)$ denote its weighted row sum at node $n$. The two transitional terms are:
\begin{align}
    (L_{in}x)_{n,t} &= \hat{d}(n) (F^{in}_{n,t})^\top F^{in}_{n,t} x_{n,t} - \sum_{m} \hat{A}(n,m) (F^{in}_{n,t})^\top F^{out}_{m,t-1} x_{m,t-1}, \label{eq:lin}\\
    (L_{out}x)_{n,t} &= \hat{d}(n) (F^{out}_{n,t})^\top F^{out}_{n,t} x_{n,t} - \sum_{m} \hat{A}(n,m) (F^{out}_{n,t})^\top F^{in}_{m,t+1} x_{m,t+1},\label{eq:lout}
\end{align}
each masked to zero at the first and last timestep, where no earlier or later patch exists. The full Laplacian is then:
\begin{equation}
    L = L_{sp} + L_{in} + L_{out}, 
\end{equation}
a linear operator on $\mathbb{R}^{N\times P\times d}$ that we equivalently view as a block matrix with one $d\times d$ block $L_{(n,t),(m,s)}$ per pair of joint-graph nodes, acting on $x_{m,s}\in\mathbb{R}^d$ (Figure~\ref{fig:sheaf}).

\begin{proposition}
\label{prop:symmetric}
$L$ is a real, symmetric operator on $\mathcal{G}^\star$: $L_{(n,t),(m,s)}^\top = L_{(m,s),(n,t)}$ for every pair of joint-graph nodes, even though $F^{\mathrm{in}} \neq F^{\mathrm{out}}$ in general.
\end{proposition}

\textbf{Proof.} Every diagonal block is a sum of Gram matrices, hence symmetric. For $s=t$ (spatial neighbors), Eq.~\ref{eq:lsp} gives $L_{(n,t),(m,t)} = -\tilde{A}(n,m)V_{n,t}^\top V_{m,t}$; since $(V_{n,t}^\top V_{m,t})^\top = V_{m,t}^\top V_{n,t}$ and $\tilde{A}(n,m)=\tilde{A}(m,n)$, this equals $L_{(m,t),(n,t)}^\top$. For $s=t-1$, Eq.~\ref{eq:lin} gives $L_{(n,t),(m,t-1)} = -\hat{A}(n,m)(F^{\mathrm{in}}_{n,t})^\top F^{\mathrm{out}}_{m,t-1}$, while Eq.~\ref{eq:lout} evaluated at node $m$, time $t-1$, gives $L_{(m,t-1),(n,t)} = -\hat{A}(m,n)(F^{\mathrm{out}}_{m,t-1})^\top F^{\mathrm{in}}_{n,t}$. Transposing the first block, $\big(-\hat{A}(n,m)(F^{\mathrm{in}}_{n,t})^\top F^{\mathrm{out}}_{m,t-1}\big)^\top = -\hat{A}(n,m)(F^{\mathrm{out}}_{m,t-1})^\top F^{\mathrm{in}}_{n,t}$, which equals $L_{(m,t-1),(n,t)}$ since $\hat{A}(n,m)=\hat{A}(m,n)$. Symmetry holds block by block, and therefore for $L$ as a whole; it follows from $L_{\mathrm{in}}$'s backward-looking mapping and $L_{\mathrm{out}}$'s forward-looking mapping being adjoint to one another. $\square$

\begin{proposition}[Positive semi-definiteness]
\label{prop:psd}
$L = L_{\mathrm{sp}} + L_{\mathrm{in}} + L_{\mathrm{out}}$ is positive semi-definite.
\end{proposition}

\textbf{Proof.} \emph{Spatial term.} Fix a patch $t$ and let $y_{n,t} := V_{n,t}x_{n,t}\in\mathbb{R}^r$, with components $y_{n,t,k}$, $k=1,\dots,r$. Expanding Eq.~\ref{eq:lsp},
\[
x^\top L_{\mathrm{sp}} x = \sum_{t,k} y_{\cdot,t,k}^\top\big(\mathrm{diag}(\tilde A\mathbf{1}) - \tilde{A}\big)y_{\cdot,t,k} = \sum_{t,k}\tfrac12\sum_{n,m}\tilde{A}(n,m)\big(y_{n,t,k}-y_{m,t,k}\big)^2 \ge 0,
\]
using that $\mathrm{diag}(W\mathbf{1}) - W$ is positive semi-definite for any symmetric $W\ge0$.

\emph{Temporal terms.} Let $u_{n,t} := F^{\mathrm{in}}_{n,t}x_{n,t}$ and $w_{n,t} := F^{\mathrm{out}}_{n,t}x_{n,t}$, both in $\mathbb{R}^r$. Expanding Eq.~\ref{eq:lin}-Eq.\ref{eq:lout},
\[
x^\top L_{\mathrm{in}} x = \sum_{n,t} \hat{d}(n)\|u_{n,t}\|^2 - \sum_{n,t,m} \hat{A}(n,m)\,u_{n,t}^\top w_{m,t-1},
\]
\[
x^\top L_{\mathrm{out}} x = \sum_{n,t} \hat{d}(n)\|w_{n,t}\|^2 - \sum_{n,t,m} \hat{A}(n,m)\,w_{n,t}^\top u_{m,t+1}.
\]
Reindexing the second cross term ($n\leftrightarrow m$, $t\mapsto t-1$) and using $\hat{A}(n,m)=\hat{A}(m,n)$ shows the two cross terms are equal; reindexing $\sum_{n,t}\hat{d}(n)\|w_{n,t}\|^2$ the same way aligns every diagonal term with a cross term at the same $(n,t,m)$. Completing the square gives the exact identity
\[
x^\top(L_{in} + L_{out})x = \sum_{n,t,m} \hat{A}(n, m) \| F^{in}_{n,t} x_{n,t} - F^{out}_{m,t-1} x_{m,t-1} \|^2_2,
\]
the sum ranging over all valid $(n,t,m)$ with $t\ge2$; every term is nonnegative since $\hat{A}(n,m)\ge0$. Summing the two nonnegative quadratic forms gives $x^\top L x \ge 0$ for all $x$. Figure~\ref{fig:trm-example} illustrates how a node's past affect its neighbor's present. $\square$

\begin{figure}[!ht]
    \centering
    \includegraphics[width=0.89\linewidth]{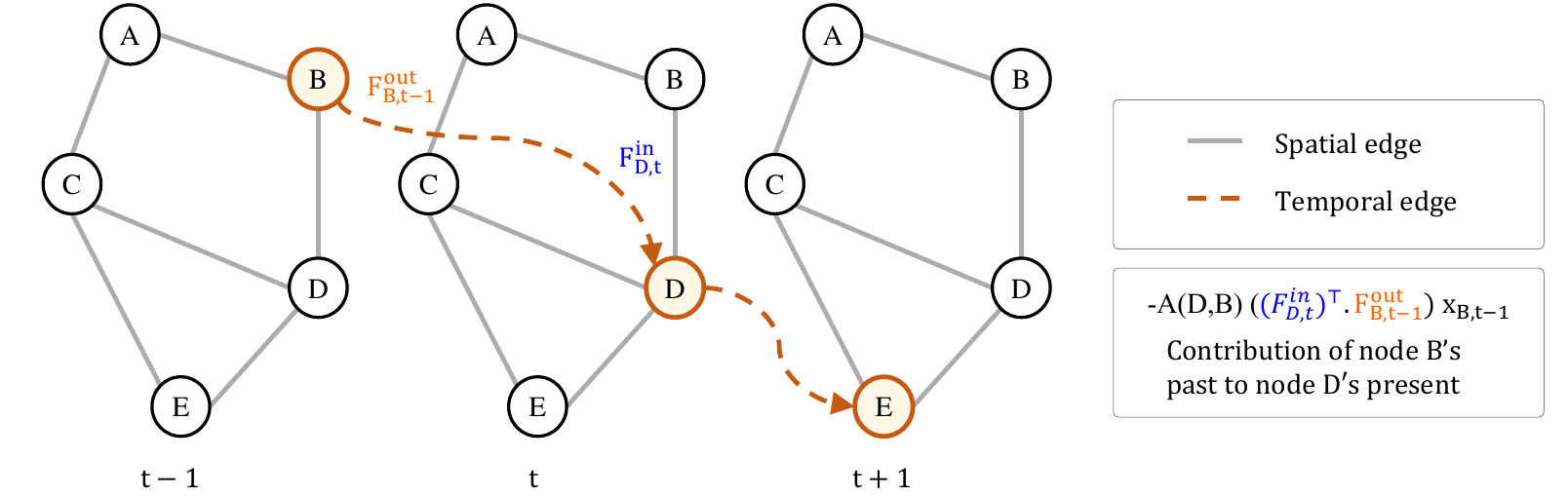}
    \caption{Node B's signal at $t-1$ contributes $-A(D,B)(F^{\mathrm{in}}_{D,t})^\top F^{\mathrm{out}}_{B,t-1}\,x_{B,t-1}$ to $(Lx)_{D,t}$, transporting B's signal from its outgoing frame into D's incoming frame. The analogous term one step later carries D's present into E's future.}
    \label{fig:trm-example}
\end{figure}

\subsection{Dual Spectral Filter (DSF)}
\label{sec:dsf}

ChronoSheaf's second component filters $L$ with the Chebyshev polynomial of Section~\ref{sec:prelim}, using per-channel coefficients $\theta_{c,0:K}$: the output is $\sum_k \big(T_k(\tilde{L})x\big)\,\Theta_k$, where $\Theta_k=\mathrm{diag}(\theta_{1,k},\dots,\theta_{d,k})$ scales the $d$ channels, so one order-$K$ recursion over $L$ filters jointly across the spatial and temporal edges of $\mathcal{E}^\star$. We call this the \emph{dual spectral filter}: a single Chebyshev recursion that is simultaneously a spatial and a temporal filter, since Propositions~\ref{prop:symmetric} and \ref{prop:psd} let $L$ carry both kinds of edges as one real, symmetric, positive semi-definite operator.

Applying this filter requires an estimate of the sheaf Laplacian's largest eigenvalue $\lambda_{\max}(L)$. Because $V, F^{\mathrm{in}}, F^{\mathrm{out}}$ are learned and unconstrained in norm, no fixed bound holds for $L$ in general, and an eigendecomposition at every forward pass is too costly. We instead adaptively bound $\lambda_{\max}(L)$: for a real symmetric block matrix with $d\times d$ blocks $L_{(n,t),(m,s)}$,
\begin{equation}
\lambda_{\max}(L) \;\le\; \max_{(n,t)} \Big[ \big\|L_{(n,t),(n,t)}\big\|_2 + \sum_{(m,s)\neq(n,t)} \big\|L_{(n,t),(m,s)}\big\|_2 \Big],
\label{eq:block-gershgorin}
\end{equation}
where $\|\cdot\|_2$ is the spectral norm. This is a generalization of the known scalar Gershgorin bound (entrywise absolute value) to the $d\times d$ blocks of $L$. We bound Eq.~\ref{eq:block-gershgorin} directly from the frames of Section~\ref{sec:trm}, using the triangle and submultiplicative inequalities for the spectral norm ($\|A+B\|_2\le\|A\|_2+\|B\|_2$, $\|AB\|_2\le\|A\|_2\|B\|_2$, $\|A^\top A\|_2 = \|A\|_2^2$):
\begin{equation}
\begin{aligned}
\hat\lambda_{\max} := \max_{(n,t)} \Big[\, &s_n\|V_{n,t}\|_2^2 + \hat{d}(n)\big(\|F^{\mathrm{in}}_{n,t}\|_2^2 + \|F^{\mathrm{out}}_{n,t}\|_2^2\big) \\
&+ \sum_m \tilde{A}(n,m)\|V_{n,t}\|_2\|V_{m,t}\|_2 \\
&+ \sum_m \hat{A}(n,m)\big(\|F^{\mathrm{in}}_{n,t}\|_2\|F^{\mathrm{out}}_{m,t-1}\|_2 + \|F^{\mathrm{out}}_{n,t}\|_2\|F^{\mathrm{in}}_{m,t+1}\|_2\big) \Big],
\end{aligned}
\label{eq:gershgorin}
\end{equation}
with $s_n=\sum_m\tilde{A}(n,m)$. Every quantity in Eq.~\ref{eq:gershgorin} is a spectral norm of one $r\times d$ frame, so $\hat\lambda_{\max}$ is tracked once per forward pass at negligible cost and detached from the computation graph (no backpropagation through $\hat\lambda_{\max}$).

\begin{proposition}[Receptive field]
\label{prop:receptive-field}
The output at $(n,t)$ after $M$ stacked ChronoSheaf layers depends on the input at $(m,s)$ with $d_{\mathcal{G}}(n,m) + |t-s| \le 2MK$, where $d_{\mathcal{G}}$ is graph-hop distance in $\mathcal{G}$ and $t,s$ index patches.
\end{proposition}

\textbf{Proof.} Every nonzero block of $L = L_{\mathrm{sp}} + L_{\mathrm{in}} + L_{\mathrm{out}}$ connects $(n,t)$ to some $(m,s)$ with $d_{\mathcal{G}}(n,m)+|t-s|\le 2$ because 
$L_{\mathrm{sp}}$ contributes same-patch spatial neighbors ($d_{\mathcal{G}}(n,m)\le1$, $|t-s|=0$); $L_{\mathrm{in}}$ and $L_{\mathrm{out}}$ each contribute a self term one patch earlier or later ($d_{\mathcal{G}}(n,m)=0$, $|t-s|=1$) and a spatial-neighbor term one patch earlier or later ($d_{\mathcal{G}}(n,m)=1$, $|t-s|=1$, since $\hat{A}(n,m)\neq0$ whenever $m=n$ or $m$ is a spatial neighbor of $n$). 

The order-$K$ Chebyshev recursion of the dual spectral filter is a degree-$K$ polynomial in $L$, so one ChronoSheaf layer's output at $(n,t)$ depends on $(m,s)$ only if $d_{\mathcal{G}}(n,m)+|t-s|\le 2K$. Composing $M$ layers sums this bound $M$ times. $\square$

\subsection{ChronoSheaf Layer}
\label{sec:layer-network}

One ChronoSheaf layer applies the dual spectral filter of Section~\ref{sec:dsf} to the operator $L$ of Section~\ref{sec:laplacian}, wrapped in a residual connection,
\begin{equation}
\mathrm{ChronoSheaf}(h) = \mathrm{LayerNorm}\big(h + \mathrm{Dropout}(W\,g_\theta(L)\,h)\big), \qquad W\in\mathbb{R}^{d\times d},
\label{eq:chronosheaf-layer}
\end{equation}
with the frames $V, F^{\mathrm{in}}, F^{\mathrm{out}}$ (hence $L$) and the coefficients $\theta_{c,0:K}$ learned independently at every layer.
The full network stacks $M$ such layers. Writing $x$ for the patchified, preprocessed input and $e\in\mathbb{R}^{N\times d}$ for a learned per-node embedding broadcast across the batch and every patch, the first layer's input is $h^{(0)} = \mathrm{PE}\big(xW_{\mathrm{in}} + e\big)$,
where $W_{\mathrm{in}}$ projects into the model dimension $d$ and $\mathrm{PE}$ adds a fixed sinusoidal encoding along the patch axis to mark the order in which the temporal edges of $\mathcal{G}^\star$ connect. Each of the $M$ layers then maps $h^{(l-1)}\mapsto h^{(l)}$ as in Eq.~\ref{eq:chronosheaf-layer}.

\section{Experiments}
\label{sec:exp}

We evaluate the performance of ChronoSheaf on several spatio-temporal forecasting benchmarks. In Section~\ref{sec:exp_trm}, we assess ChronoSheaf's ability to mitigate oversmoothing in spatio-temporal learning tasks. In Section~\ref{sec:exp_forecast}, we include the results of traffic, weather, air quality forecasting on six benchmark datasets compared to SOTA baselines. Section~\ref{sec:exp_dsf} examines the effect of spectral filtering on the space-time sheaf Laplacian and how different values of polynomial orders affect performance. Finally, in Section~\ref{sec:exp_multimodal}, we examine the network's ability to retain local representations given heterogeneous features from another modality. Extended results and complete dataset details are provided in the Appendices.

\subsection{Oversmoothing Analysis}
\label{sec:exp_trm}
\begin{wrapfigure}{r}{0.53\textwidth}
\vspace{-17pt}
    \centering
    \includegraphics[width=\linewidth]{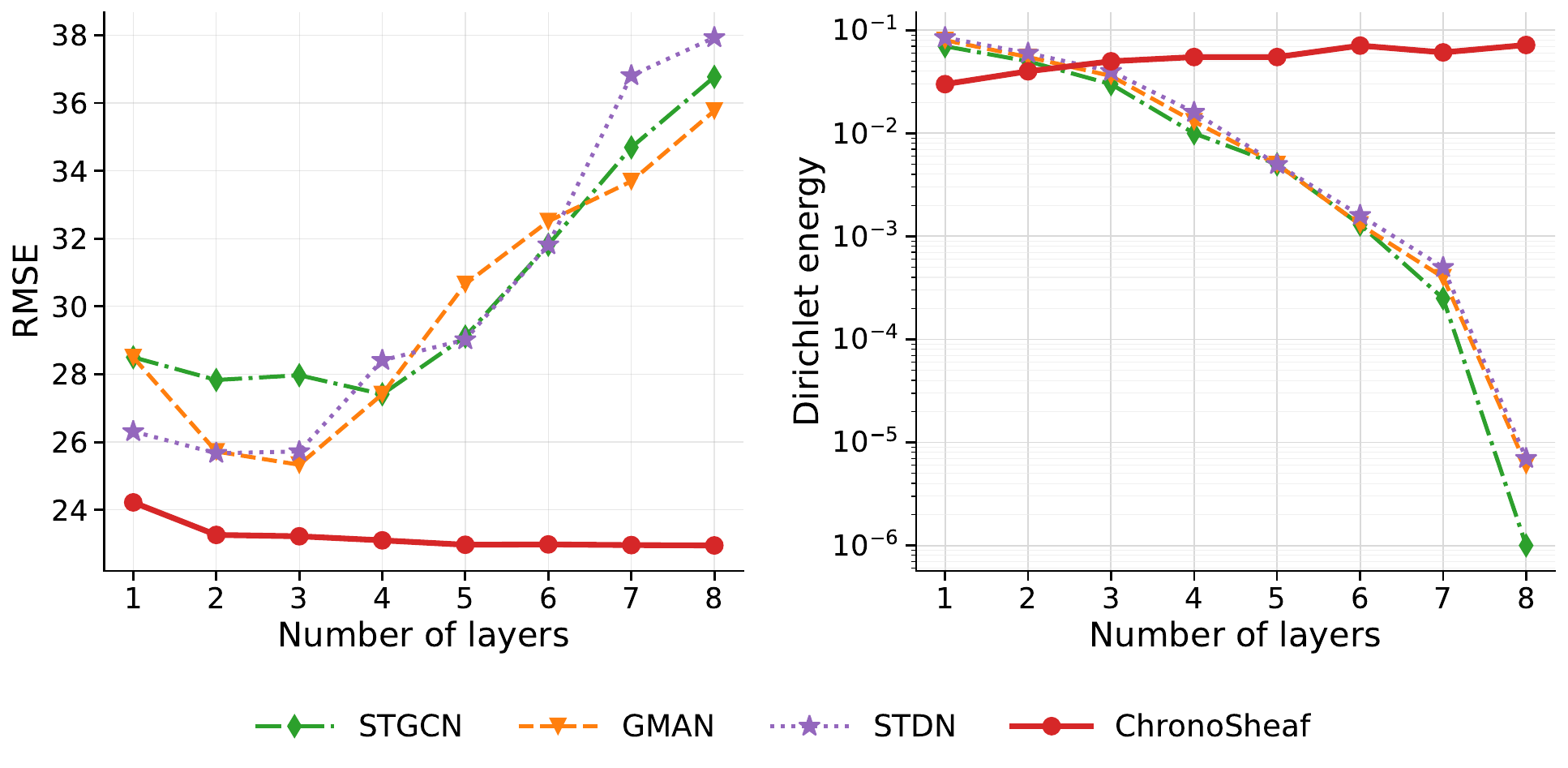}
    \caption{\small Test RMSE (left) and Dirichlet energy (right, log scale) versus
    network depth. Baseline energy collapses with depth, while
    ChronoSheaf's remains stable.}
    \label{fig:depth}
    \vspace{-60pt}
\end{wrapfigure}
We examine whether TRM preserves local representations as the network deepens. On PEMS08,
we stack $M \in \{1, \dots, 8\}$ layers and track both test Root Mean Square Error (RMSE) and the normalized Dirichlet energy:
\begin{equation}
    E(\mathbf{H}) = \frac{N}{|\mathcal{E}|} \cdot
    \frac{\displaystyle\sum_{(i,j)\in\mathcal{E}} \sum_{t=1}^{T}
          \big\lVert \mathbf{h}_{i,t} - \mathbf{h}_{j,t} \big\rVert_2^2}
         {\displaystyle\sum_{i=1}^{N} \sum_{t=1}^{T}
          \big\lVert \mathbf{h}_{i,t} \big\rVert_2^2},
\end{equation}

of the hidden node representations; energy that decays toward zero means node features are becoming indistinguishable. We compare against three representative baselines:
STGCN (convolution-based), GMAN (attention-based), and STDN (the current
state of the art in spatio-temporal forecasting). Results are shown in
Figure~\ref{fig:depth}.

All three baselines degrade sharply with depth, tracking a Dirichlet energy that collapses by four to six orders of magnitude by the eighth layer. This is exactly the known behavior of oversmoothing under repeatedly applied propagation discussed in~\cite{li2019deepgcns} and~\cite{arroyo2026vanishing}. ChronoSheaf's energy stays within roughly one order of magnitude across the same range, and its RMSE saturates as depth increases. 

\subsection{Spatio-temporal Forecasting}
\label{sec:exp_forecast}
\textbf{Setup.} For traffic forecasting, we use PEMS04 and PEMS08 \citep{song2020spatial}, in addition to the San Diego (SD) and full California (CA) subsets of LargeST~\citep{liu2023largest}. All four datasets consist of traffic flow readings from loop detectors aggregated into 5-minute intervals, with sensors as nodes and edges derived from road-network connectivity. Graph sizes vary from $170$ to $8{,}600$ nodes.
Following the experimental setup of~\cite{liu2023largest}, we use 12 historical timesteps as input to predict the next 12 steps. Further details and hyperparameter configurations are provided in Appendix~\ref{app:exp}.
Additionally, we use the Molene dataset~\citep{girault2015stationary} for weather forecasting, which contains hourly temperature readings from 32 stations in Brittany, France, with 10 historical timesteps as input to predict the next 5 steps. We further evaluate ChronoSheaf on the Air Quality dataset~\citep{urbanair_microsoft_research}, which contains hourly $\text{PM}_{2.5}$ concentration readings from 437 monitoring stations across 43 Chinese cities. The input spans 48 hours and predicts the next 48 hours. Full dataset details are provided in Appendix~\ref{app:exp}.

\textbf{Baselines.} We compare ChronoSheaf to established GCN/RNN methods (STGCN~\citep{yu2018spatio}, DCRNN~\citep{li2018diffusion}), adaptive graph learning (Graph WaveNet~\citep{wu2019graph}, CrossGNN~\citep{huang2023crossgnn}), attention-based models (ASTGCN~\citep{guo2019attention}, GMAN~\citep{zheng2020gman}), scalable architectures (SGP~\citep{cini2023scalable}, SparseTSF~\citep{pmlr-v235-lin24n}, FITS \citep{xu2024fits}), spectral methods (CITRUS~\citep{einizade2024continuous}), and the most recent decomposition-based approaches (TimeBase~\citep{huang2025timebase}, STDN~\citep{cao2025spatiotemporal}).
Note on comparison with other Sheaf-based methods: to the best of our knowledge, there is currently no published literature on sheaf-based methods for spatio-temporal modeling; all identified sheaf networks operate exclusively in the spatial domain. Therefore, direct comparisons with analogous sheaf models are not applicable.
\begin{table}[tbp]
\centering
\small
\caption{Traffic forecasting results on LargeST (SD and CA), PEMS04, and PEMS08, averaged over 12 forecasting steps. Best results are in bold and second-best are underlined. "OOM" denotes out of memory.}
\label{tab:res1}
\resizebox{0.99\linewidth}{!}{%
\begin{tabular}{lcccccccccccc}
\toprule
 \multirow{2}{*}{\textbf{Model}} & \multicolumn{3}{c}{\textbf{LargeST (SD)}} & \multicolumn{3}{c}{\textbf{LargeST (CA)}} & \multicolumn{3}{c}{\textbf{PEMS04}} & \multicolumn{3}{c}{\textbf{PEMS08}}\\ 
\cmidrule{2-13}
 & \footnotesize MAE & \footnotesize RMSE & \footnotesize MAPE
 & \footnotesize MAE & \footnotesize RMSE & \footnotesize MAPE & \footnotesize MAE & \footnotesize RMSE & \footnotesize MAPE & \footnotesize MAE & \footnotesize RMSE & \footnotesize MAPE\\
\midrule
STGCN \cite{yu2018spatio}     &  19.67 & 34.14 & 13.86 & 21.33 &  36.39 & 16.53 & 22.70 & 35.55 &  \underline{14.59} & 18.02 & 27.83 &  11.40\\
DCRNN \cite{li2018diffusion}   & 21.03 & 33.37 & 14.13 & 21.87 & 34.41 & 17.06 & 24.70 & 38.12 & 17.12 & 17.86 & 27.83 & 11.45 \\
GW-Net \cite{wu2019graph}     &  \underline{17.74} & \underline{29.62} &  \underline{11.88} & 21.72 & 34.20 & 17.40 & 25.45 & 39.70 & 17.29 & 19.13 & 31.05 & 12.68 \\
GMAN \cite{zheng2020gman}      & 19.63 & 30.01 & 12.20 & OOM & OOM & OOM &  20.23 &  \underline{32.17} & 17.06 &  16.47 & 25.72 &  \underline{10.48}\\
ASTGCN \cite{guo2019attention}  & 23.70 & 37.63 & 15.65 & OOM & OOM & OOM & 22.80 & 35.82 & 16.56 & 18.63 & 28.27 & 13.08 \\
CrossGNN \cite{huang2023crossgnn}  & 23.09 & 39.04 & 15.87 & 26.16 & 41.25 & 20.25 & 25.11 & 40.32 & 20.88 & 21.91 & 35.69 & 16.73 \\
FITS \cite{xu2024fits}  & 21.34 & 38.01 & 14.36 & 20.12 & 35.22 & 15.81 & 20.02 & 33.31 & 16.49 & \underline{15.43} & 26.34 & 10.75 \\
SGP \cite{cini2023scalable}     & 20.25 & 35.06 & 12.33 & OOM & OOM & OOM &21.90 &34.27 & 16.15& 18.06&27.33&13.55\\
CITRUS \cite{einizade2024continuous}  & 18.22 & 30.67 & 12.23 & \underline{19.66} & 33.50 & 15.72 & 22.64 &35.08 & 16.83&17.89&27.81&11.21\\
SparseTSF \cite{pmlr-v235-lin24n} & 25.90 & 46.77 & 18.30 & 23.75 & 43.03 & 18.87 & 22.24 & 38.65 & 17.25 & 16.61 & 30.84 & 11.54 \\
TimeBase \cite{huang2025timebase}  & 19.15 & 32.21 & 13.22 & 20.23 & \underline{32.43} & \underline{14.42} & 23.21 & 41.59 & 16.18 & 17.52 & 35.99 & 11.98 \\
STDN \cite{cao2025spatiotemporal} & 19.01 & 31.25 & 14.10 & OOM & OOM & OOM & \textbf{19.15} & 34.37 & 23.51 &18.26& \underline{25.67}&20.93\\
\midrule
\textbf{ChronoSheaf}  & \textbf{16.42 }& \textbf{28.33} & \textbf{10.72 }& \textbf{17.52} & \textbf{31.66 }& \textbf{14.22} & \underline{19.95} & \textbf{29.11} & \textbf{13.11} & \textbf{14.32} & \textbf{23.26} & \textbf{10.21 }\\
\bottomrule
\end{tabular}%
}
\end{table}

\textbf{Results.} The reported results in Table~\ref{tab:res1} for traffic forecasting are in Mean Absolute Error (MAE), Root Mean Square Error (RMSE), and Mean Absolute Percentage Error (MAPE), averaged over 12 forecasting steps. On the four datasets, ChronoSheaf achieves the best result in 11 of the 12 reported dataset–metric combinations. Table~\ref{tab:res2} reports MAE for weather and air quality forecasting across multiple prediction horizons. ChronoSheaf's advantage over the baselines is most pronounced at longer horizons, which is consistent with its ability to capture extended temporal context.

\begin{table}[htbp]
\centering
\caption{Weather (Molene) and air quality (PM$_{2.5}$) forecasting results in MAE across multiple prediction horizons. Best results are in bold and second-best are underlined.}
\label{tab:res2}
\resizebox{0.8\linewidth}{!}{%
\begin{tabular}{lcccccccc}
\toprule
 \multirow{2}{*}{\textbf{Model}} & \multicolumn{5}{c}{\textbf{Molene (weather)}} & \multicolumn{3}{c}{\textbf{Air Quality ($\text{PM}_{2.5}$)}}\\ 
\cmidrule{2-9}
 & \textit{step-1} & \textit{step-2} & \textit{step-3} & \textit{step-4} & \textit{step-5} 
 & \textit{1-12 hr} & \textit{13-24 hr} & \textit{25-48 hr}\\
\midrule
STGCN \citep{yu2018spatio}     & 1.066& 1.552 & 1.933&2.225 & 2.698 & 22.77 & 31.59 & 35.59\\
DCRNN \citep{li2018diffusion}   & 0.396&  0.635 &  \underline{0.907} &  \underline{1.164} &  \underline{1.419} & 24.28 & 32.95 & 36.11\\
GW-Net \citep{wu2019graph}      & 0.855& 1.124 & 1.455&1.865 & 2.104 & 22.42 & 30.29 & 34.73 \\
GMAN \citep{zheng2020gman}      & 0.962& 1.424 & 1.882& 2.022& 2.417 & 24.11 & 32.80& 36.25\\
ASTGCN \citep{guo2019attention}  & 0.981 & 1.342 & 1.877 &  1.906& 2.125 & 25.51 & 33.98 & 37.42\\
CrossGNN \citep{huang2023crossgnn}   & 0.952 & 1.244 & 1.432 & 1.880 & 2.025 & 25.63 & 33.74 & 36.96\\
FITS \citep{xu2024fits}   &  0.813 & 1.036  & 1.37 & 1.671 & 1.964 & 21.81 & 31.17 & 34.56\\
SGP \citep{cini2023scalable}     &  0.522 & 0.848 & 1.148 & 1.399 & 1.650 & 26.42 &33.38 & 35.99\\
CITRUS \citep{einizade2024continuous}  &  0.497 & 0.757  & 1.048 & 1.332& 1.600 & 27.40 & 32.66 & 35.56\\
SparseTSF \citep{pmlr-v235-lin24n}  & 0.882 & 1.137 & 1.269 & 1.566 & 1.897 & 23.44 & 32.65 & 35.51\\
TimeBase \citep{huang2025timebase}  & 0.574 & 0.821 & 1.036 & 1.478 & 1.988 & \textbf{21.20} & \textbf{29.88} & \underline{32.97}\\
STDN \citep{cao2025spatiotemporal} & \underline{0.352} &  \underline{0.621} & 0.924 &  1.336 & 1.566 & 22.17 &  30.16 &  34.62\\
\midrule
\textbf{ChronoSheaf}  & \textbf{ 0.312} & \textbf{0.524} & \textbf{0.806} & \textbf{1.010} & \textbf{1.242} & \underline{21.31} & \underline{29.90} & \textbf{32.41} \\
\bottomrule
\end{tabular}%
}
\end{table}

\subsection{Spectral Filtering Analysis}
\label{sec:exp_dsf}
We vary the order $k$ of the polynomial filter on the space-time sheaf
Laplacian, which sets the number of hops each layer aggregates over on the PEMS08 dataset (Table~\ref{tab:k_comparison}). Accuracy improves sharply up to $k=4$ and then saturates. Training time grows almost linearly with $k$, as each order adds one Laplacian multiplication.
It should be noted that the spectral filtering approach here inherits the same advantages of the sheaf Laplacian in being more expressive. We therefore select $k$ per dataset to balance predictive accuracy against computational cost, taking into account both graph size and the marginal gain from higher orders. Accordingly, we set $k=2$ for the LargeST dataset and $k=4$ for all other datasets.

\begin{table}[tbp]
\centering
\caption{Effect of the spectral filter order $k$ on test RMSE and training
time.}
\label{tab:k_comparison}
\resizebox{0.9\linewidth}{!}{%
\begin{tabular}{lcccccccccc}
\toprule
Polynomial Filter Order & $k=1$ & $k=2$ & $k=3$ & $k=4$ & $k=5$ & $k=6$ & $k=7$ & $k=8$ & $k=9$ & $k=10$ \\
\midrule
RMSE & 25.96 & 23.92 & 23.26 & 23.20 & 23.15 & 23.15 & 23.05 & 23.04 & 23.01 & 23.01 \\
Training Time (s/epoch) & 8.021 & 11.63 & 15.22 & 19.46 & 23.86 & 26.80 & 30.57 & 34.80 & 37.70 & 40.32 \\
\bottomrule
\end{tabular}}
\end{table}

\subsection{Performance under heterogeneous features}
\label{sec:exp_multimodal}
\begin{wraptable}{r}{0.5\textwidth}
\centering
\caption{\small Traffic forecasting results on PEMS04 and PEMS08 augmented with NOAA weather covariates from the corresponding regions, averaged over 12 forecasting steps. Best results are in bold and second-best are underlined.}
\label{tab:res_multi}
\resizebox{\linewidth}{!}{%
\begin{tabular}{lcccccc}
\toprule
 \multirow{2}{*}{\textbf{Model}} & \multicolumn{3}{c}{\textbf{Multimodal PEMS04}} & \multicolumn{3}{c}{\textbf{Multimodal PEMS08}}\\ 
\cmidrule{2-7}
 & \footnotesize MAE & \footnotesize RMSE & \footnotesize MAPE
 & \footnotesize MAE & \footnotesize RMSE & \footnotesize MAPE \\
\midrule
STGCN \citep{yu2018spatio}  & 21.36 & 34.42 & 13.25 & 17.98 & 27.52 & \underline{11.10} \\
DCRNN \citep{li2018diffusion}  & 21.67 & 35.91 & 14.32 & 18.44 & 27.93 & 13.62 \\
GW-Net \citep{wu2019graph}     & 24.17 & 38.22 & 16.54 & 18.88 & 28.79 & 14.02 \\
GMAN \citep{zheng2020gman}      & 23.63 & 37.67 & 15.50 & 17.08 & 28.41 & 12.66 \\
CITRUS \citep{einizade2024continuous} & 20.34 & 30.11 & 14.89 & 15.56 & 27.31 &  12.10\\
STDN \citep{cao2025spatiotemporal}  & \underline{18.33} & \underline{29.76} & \textbf{11.10} & \underline{14.31} & \underline{24.01} & 11.63 \\
\midrule
\textbf{ChronoSheaf} & \textbf{15.53 }& \textbf{27.22}& \underline{12.47}& \textbf{13.26} & \textbf{22.59 }&\textbf{ 10.55} \\
\bottomrule
\end{tabular}%
}
\end{wraptable}
Real-world sensor networks rarely observe a single, homogeneous signal: traffic dynamics are modulated by exogenous factors, which are recorded by physically distinct stations on different scales and with different semantics. Standard graph diffusion treats all node features as residing in a common space and propagates them through a shared aggregation, which tends to wash out modality-specific information and over-smooth local representations when heterogeneous signals are mixed. To test this, we augment PEMS04 and PEMS08 with meteorological covariates from NOAA weather stations across the same regions (District 4 and District 8) during the same time intervals, and assess whether each model can exploit the additional modality. Baselines whose architectures do not support additional exogenous node features are omitted from this comparison. As shown in Table~\ref{tab:res_multi}, ChronoSheaf achieves the lowest MAE and RMSE on both datasets, indicating that ChronoSheaf can effectively exploit heterogeneous input features without degrading forecasting performance. 

\section{Conclusion}
In this work, we introduced ChronoSheaf, a cellular sheaf defined jointly over space and time, consisting of two main components. First, transitional restriction maps decompose each node's restriction map into a spatial frame and a pair of independently gated frames that map it to its neighbors at the preceding and following timestep, giving a node asymmetric sensitivity to past versus future influence. We proved the resulting joint Laplacian is real, symmetric, and positive semi-definite by construction. Second, the dual spectral filter applies a single polynomial recursion to this operator, with an adaptively estimated spectral radius, giving one filtering pass a receptive field that grows jointly in graph hops and temporal lag. Evaluated on several forecasting benchmarks, ChronoSheaf outperforms established spatio-temporal graph baselines.

\section{Ethics and Reproducibility Statement}

\textbf{Ethics Statement:} We do not foresee any immediate negative societal or ethical impacts of our work.

\textbf{Reproducibility Statement:}
We are committed to ensuring the reproducibility of our results. Detailed experimental setup and dataset descriptions are provided in the appendices. The source code is available at: \url{https://anonymous.4open.science/r/ChronoSheaf-1D7D/}.

\section{AI use statement}
In this work, we used generative AI tools for: aiding and polishing writing, grammar and language checks, reviewing the paper before submission, providing feedback on research methodology and experiments, assisting with translation, summarizing existing literature, and editing the research paper to improve readability.
We have not used generative AI tools for: ideation, developing theoretical models or conceptual frameworks, formulating mathematical claims, generating synthetic data sets, providing critical ingredients for proving mathematical claims, assisting in the writing of proofs, proposing or refining hypotheses, cleaning and reformatting datasets, supporting qualitative and thematic data analysis, or interpreting results. 
We have reviewed all AI-assisted work. We take responsibility for the final content of this work.

\bibliography{iclr2027_conference}

@article{defferrard2016convolutional,
  title={Convolutional neural networks on graphs with fast localized spectral filtering},
  author={Defferrard, Micha{\"e}l and Bresson, Xavier and Vandergheynst, Pierre},
  journal={Advances in neural information processing systems},
  volume={29},
  year={2016}
}

@article{einizade2024continuous,
  title={Continuous product graph neural networks},
  author={Einizade, Aref and Malliaros, Fragkiskos and Giraldo, Jhony H},
  journal={Advances in Neural Information Processing Systems},
  volume={37},
  pages={90226--90252},
  year={2024}
}

@inproceedings{cao2025spatiotemporal,
  title={Spatiotemporal-aware trend-seasonality decomposition network for traffic flow forecasting},
  author={Cao, Lingxiao and Wang, Bin and Jiang, Guiyuan and Yu, Yanwei and Dong, Junyu},
  booktitle={Proceedings of the AAAI Conference on Artificial Intelligence},
  volume={39},
  pages={11463--11471},
  year={2025}
}

@inproceedings{cini2023scalable,
  title={Scalable spatiotemporal graph neural networks},
  author={Cini, Andrea and Marisca, Ivan and Bianchi, Filippo Maria and Alippi, Cesare},
  booktitle={Proceedings of the AAAI conference on artificial intelligence},
  volume={37},
  pages={7218--7226},
  year={2023}
}

@article{huang2023crossgnn,
  title={Crossgnn: Confronting noisy multivariate time series via cross interaction refinement},
  author={Huang, Qihe and Shen, Lei and Zhang, Ruixin and Ding, Shouhong and Wang, Binwu and Zhou, Zhengyang and Wang, Yang},
  journal={Advances in Neural Information Processing Systems},
  volume={36},
  pages={46885--46902},
  year={2023}
}

@InProceedings{pmlr-v235-lin24n,
  title = 	 {{S}parse{TSF}: Modeling Long-term Time Series Forecasting with *1k* Parameters},
  author =       {Lin, Shengsheng and Lin, Weiwei and Wu, Wentai and Chen, Haojun and Yang, Junjie},
  booktitle = 	 {Proceedings of the 41st International Conference on Machine Learning},
  pages = 	 {30211--30226},
  year = 	 {2024},
  volume = 	 {235},
  series = 	 {Proceedings of Machine Learning Research},
  month = 	 {21--27 Jul},
  publisher =    {PMLR}
}

@inproceedings{huang2025timebase,
  title     = {TimeBase: The Power of Minimalism in Efficient Long-Term Time Series Forecasting},
  author    = {Huang, Qihe and Zhou, Zhengyang and Yang, Kuo and Yi, Zhongchao and Wang, Xu and Wang, Yang},
  booktitle = {Proceedings of the Forty-Second International Conference on Machine Learning},
  year      = {2025}
}

@inproceedings{xu2024fits,
  title={FITS: Modeling time series with $10 k $ parameters},
  author={Xu, Zhijian and Zeng, Ailing and Xu, Qiang},
  booktitle={International Conference on Learning Representations},
  volume={2024},
  pages={26295--26318},
  year={2024}
}

@article{bodnar2022neural,
  title={Neural sheaf diffusion: A topological perspective on heterophily and oversmoothing in gnns},
  author={Bodnar, Cristian and Di Giovanni, Francesco and Chamberlain, Benjamin and Lio, Pietro and Bronstein, Michael},
  journal={Advances in Neural Information Processing Systems},
  volume={35},
  pages={18527--18541},
  year={2022}
}

@inproceedings{yu2018spatio,
  title={Spatio-Temporal Graph Convolutional Networks: A Deep Learning Framework for Traffic Forecasting},
  author={Yu, Bing and Yin, Haoteng and Zhu, Zhanxing},
  booktitle={Proceedings of the Twenty-Seventh International Joint Conference on Artificial Intelligence},
  pages={3634--3640},
  year={2018},
  organization={International Joint Conferences on Artificial Intelligence Organization}
}

@inproceedings{wu2019graph,
  title={Graph WaveNet for Deep Spatial-Temporal Graph Modeling},
  author={Wu, Zonghan and Pan, Shirui and Long, Guodong and Jiang, Jing and Zhang, Chengqi},
  booktitle={28th International Joint Conference on Artificial Intelligence (IJCAI-19)},
  year={2019},
  organization={International Joint Conference on Artificial Intelligence (IJCAI)}
}

@inproceedings{zheng2020gman,
  title={Gman: A graph multi-attention network for traffic prediction},
  author={Zheng, Chuanpan and Fan, Xiaoliang and Wang, Cheng and Qi, Jianzhong},
  booktitle={Proceedings of the AAAI conference on artificial intelligence},
  volume={34},
  pages={1234--1241},
  year={2020}
}

@inproceedings{wu2020connecting,
  title={Connecting the dots: Multivariate time series forecasting with graph neural networks},
  author={Wu, Zonghan and Pan, Shirui and Long, Guodong and Jiang, Jing and Chang, Xiaojun and Zhang, Chengqi},
  booktitle={Proceedings of the 26th ACM SIGKDD international conference on knowledge discovery \& data mining},
  pages={753--763},
  year={2020}
}

@article{bai2020adaptive,
  title={Adaptive graph convolutional recurrent network for traffic forecasting},
  author={Bai, Lei and Yao, Lina and Li, Can and Wang, Xianzhi and Wang, Can},
  journal={Advances in neural information processing systems},
  volume={33},
  pages={17804--17815},
  year={2020}
}

@inproceedings{ribeiro2026cooperative,
  title={Cooperative sheaf neural networks},
  author={Ribeiro, Andr{\'e} and Ten{\'o}rio, Ana Luiza and Belieni, Juan and Souza, Amauri and Mesquita, Diego},
  booktitle={International Conference on Learning Representations},
  volume={2026},
  pages={144580--144596},
  year={2026}
}

@inproceedings{chen2021z,
  title={Z-GCNETs: Time zigzags at graph convolutional networks for time series forecasting},
  author={Chen, Yuzhou and Segovia, Ignacio and Gel, Yulia R},
  booktitle={International conference on machine learning},
  pages={1684--1694},
  year={2021},
  organization={PMLR}
}

@misc{urbanair_microsoft_research,
  title        = {Urban Air},
  author       = {{Microsoft Research}},
  howpublished = {\url{https://www.microsoft.com/en-us/research/project/urban-air/}},
  year         = {2026},
  note         = {Accessed: 2026-07-04},
}

@inproceedings{girault2015stationary,
  title={Stationary graph signals using an isometric graph translation},
  author={Girault, Benjamin},
  booktitle={2015 23rd European Signal Processing Conference (EUSIPCO)},
  pages={1516--1520},
  year={2015},
  organization={IEEE}
}

@article{arroyo2026vanishing,
  title={On vanishing gradients, over-smoothing, and over-squashing in gnns: Bridging recurrent and graph learning},
  author={Arroyo, {\'A}lvaro and Gravina, Alessio and Gutteridge, Benjamin and Barbero, Federico and Gallicchio, Claudio and Dong, Xiaowen and Bronstein, Michael and Vandergheynst, Pierre},
  journal={Advances in Neural Information Processing Systems},
  volume={38},
  pages={74356--74393},
  year={2026}
}

@inproceedings{li2019deepgcns,
  title={DeepGCNs: Can GCNs go as deep as CNNs?},
  author={Li, Guohao and Muller, Matthias and Thabet, Ali and Ghanem, Bernard},
  booktitle={Proceedings of the IEEE/CVF international conference on computer vision},
  pages={9267--9276},
  year={2019}
}

@article{liu2023largest,
  title={Largest: A benchmark dataset for large-scale traffic forecasting},
  author={Liu, Xu and Xia, Yutong and Liang, Yuxuan and Hu, Junfeng and Wang, Yiwei and Bai, Lei and Huang, Chao and Liu, Zhenguang and Hooi, Bryan and Zimmermann, Roger},
  journal={Advances in neural information processing systems},
  volume={36},
  pages={75354--75371},
  year={2023}
}

@inproceedings{fiorini2026sheaves,
  title={Sheaves reloaded: A direction awakening},
  author={Fiorini, Stefano and Aktas, Hakan Emre and Duta, Iulia and Morerio, Pietro and Del Bue, Alessio and Lio, Pietro and Coniglio, Stefano},
  booktitle={International Conference on Learning Representations},
  volume={2026},
  pages={5144--5168},
  year={2026}
}

@inproceedings{song2020spatial,
  title={Spatial-temporal synchronous graph convolutional networks: A new framework for spatial-temporal network data forecasting},
  author={Song, Chao and Lin, Youfang and Guo, Shengnan and Wan, Huaiyu},
  booktitle={Proceedings of the AAAI conference on artificial intelligence},
  volume={34},
  pages={914--921},
  year={2020}
}

@article{sabbaqi2023graph,
  title={Graph-time convolutional neural networks: Architecture and theoretical analysis},
  author={Sabbaqi, Mohammad and Isufi, Elvin},
  journal={IEEE Transactions on Pattern Analysis and Machine Intelligence},
  volume={45},
  number={12},
  pages={14625--14638},
  year={2023},
  publisher={IEEE}
}

@inproceedings{li2018diffusion,
  title={Diffusion Convolutional Recurrent Neural Network: Data-Driven Traffic Forecasting},
  author={Li, Yaguang and Yu, Rose and Shahabi, Cyrus and Liu, Yan},
  booktitle={International Conference on Learning Representations},
  year={2018}
}

@inproceedings{guo2019attention,
  title={Attention based spatial-temporal graph convolutional networks for traffic flow forecasting},
  author={Guo, Shengnan and Lin, Youfang and Feng, Ning and Song, Chao and Wan, Huaiyu},
  booktitle={Proceedings of the AAAI conference on artificial intelligence},
  volume={33},
  pages={922--929},
  year={2019}
}

@inproceedings{choi2022graph,
  title={Graph neural controlled differential equations for traffic forecasting},
  author={Choi, Jeongwhan and Choi, Hwangyong and Hwang, Jeehyun and Park, Noseong},
  booktitle={Proceedings of the AAAI conference on artificial intelligence},
  volume={36},
  pages={6367--6374},
  year={2022}
}

@article{bruna2014spectral,
  title={Spectral networks and locally connected networks on graphs},
  author={Bruna, Joan and Zaremba, Wojciech and Szlam, Arthur and LeCun, Yann},
  journal={arXiv preprint arXiv:1312.6203},
  year={2013}
}

@inproceedings{kipf2017gcn,
  title     = {Semi-Supervised Classification with Graph Convolutional Networks},
  author    = {Kipf, Thomas N. and Welling, Max},
  booktitle = {International Conference on Learning Representations (ICLR)},
  year      = {2017}
}

@article{velickovic2018gat,
  title={Graph attention networks},
  author={Veli{\v{c}}kovi{\'c}, Petar and Cucurull, Guillem and Casanova, Arantxa and Romero, Adriana and Lio, Pietro and Bengio, Yoshua},
  journal={arXiv preprint arXiv:1710.10903},
  year={2017}
}

@article{rusch2023survey,
  title={A survey on oversmoothing in graph neural networks},
  author={Rusch, T Konstantin and Bronstein, Michael M and Mishra, Siddhartha},
  journal={arXiv preprint arXiv:2303.10993},
  year={2023}
}

@inproceedings{li2018oversmoothing,
  title     = {Deeper Insights into Graph Convolutional Networks for Semi-Supervised Learning},
  author    = {Li, Qimai and Han, Zhichao and Wu, Xiao-Ming},
  booktitle = {Proceedings of the AAAI Conference on Artificial Intelligence},
  year      = {2018}
}

@inproceedings{alon2021bottleneck,
  title     = {On the Bottleneck of Graph Neural Networks and its Practical Implications},
  author    = {Alon, Uri and Yahav, Eran},
  booktitle = {International Conference on Learning Representations (ICLR)},
  year      = {2021}
}

@article{hansen2019spectral,
  title={Toward a spectral theory of cellular sheaves: J. Hansen, R. Ghrist},
  author={Hansen, Jakob and Ghrist, Robert},
  journal={Journal of Applied and Computational Topology},
  volume={3},
  number={4},
  pages={315--358},
  year={2019},
  publisher={Springer}
}

@inproceedings{barbero2022connection,
  title = 	 {Sheaf Neural Networks with Connection Laplacians},
  author =       {Barbero, Federico and Bodnar, Cristian and S\'aez de Oc\'ariz Borde, Haitz and Bronstein, Michael and Veli\v{c}kovi\'c, Petar and Li\`o, Pietro},
  booktitle = 	 {Proceedings of Topological, Algebraic, and Geometric Learning Workshops 2022},
  pages = 	 {28--36},
  year = 	 {2022},
  volume = 	 {196},
  series = 	 {Proceedings of Machine Learning Research},
  month = 	 {25 Feb--22 Jul},
  publisher =    {PMLR},
}

@inproceedings{bamberger2025bundle,
  title={Bundle neural network for message diffusion on graphs},
  author={Bamberger, Jacob and Barbero, Federico and Dong, Xiaowen and Bronstein, Michael},
  booktitle={International Conference on Learning Representations},
  volume={2025},
  pages={16788--16813},
  year={2025}
}

@article{hajij2025copresheaf,
  title={Copresheaf topological neural networks: A generalized deep learning framework},
  author={Hajij, Mustafa and Bastian, Lennart and Osentoski, Sarah and Kabaria, Hardik and Davenport, John and Cherukuri, Balaji and Kocheemoolayil, Joseph and Shahmansouri, Nastaran and Lew, Adrian and Papamarkou, Theodore and others},
  journal={Advances in Neural Information Processing Systems},
  volume={38},
  pages={149759--149803},
  year={2025}
}

@inproceedings{fang2021stgode,
  title={Spatial-temporal graph ode networks for traffic flow forecasting},
  author={Fang, Zheng and Long, Qingqing and Song, Guojie and Xie, Kunqing},
  booktitle={Proceedings of the 27th ACM SIGKDD conference on knowledge discovery \& data mining},
  pages={364--373},
  year={2021}
}

@article{ma2025mage,
  title={Less but more: Linear adaptive graph learning empowering spatiotemporal forecasting},
  author={Ma, Jiaming and Wang, Binwu and Wang, Guanjun and Yang, Kuo and Zhou, Zhengyang and Wang, Pengkun and Wang, Xu and Wang, Yang},
  journal={Advances in Neural Information Processing Systems},
  volume={38},
  pages={58382--58418},
  year={2025}
}

@inproceedings{marisca2025stoversquashing,
  title     = {Over-squashing in Spatiotemporal Graph Neural Networks},
  author    = {Marisca, Ivan and Bamberger, Jacob and Alippi, Cesare and Bronstein, Michael M.},
  booktitle = {Advances in Neural Information Processing Systems (NeurIPS)},
  volume    = {38},
  pages     = {33969--33999},
  publisher = {Curran Associates, Inc.},
  year      = {2025}
}
\bibliographystyle{iclr2027_conference}

\newpage
\section*{Supplementary Material}
\addcontentsline{toc}{section}{Supplementary Material}
\appendix

\section{Additional Results}
\label{app:res}
\subsection{Predictions Comparison}
Figure~\ref{fig:cheb_pred} complements the spectral filtering analysis of Section~\ref{sec:exp_dsf} by comparing 12-step-ahead predictions of ChronoSheaf with Chebyshev orders $K=2$ and $K=10$ on nine PEMS08 test nodes. On nodes with steep rush-hour ramps (nodes 89, 103, and 144), $K=2$ visibly lags the rising and falling edges of the daily cycle and underestimates peak flow. $K=10$ follows these transitions closely and reduces MAE by a big margin. A higher order therefore lets each node draw on a wider space-time neighborhood, which is most useful when abrupt changes in flow propagate from upstream sensors before they reach the target node. 
\begin{figure}[h]
    \centering
    \includegraphics[width=\linewidth]{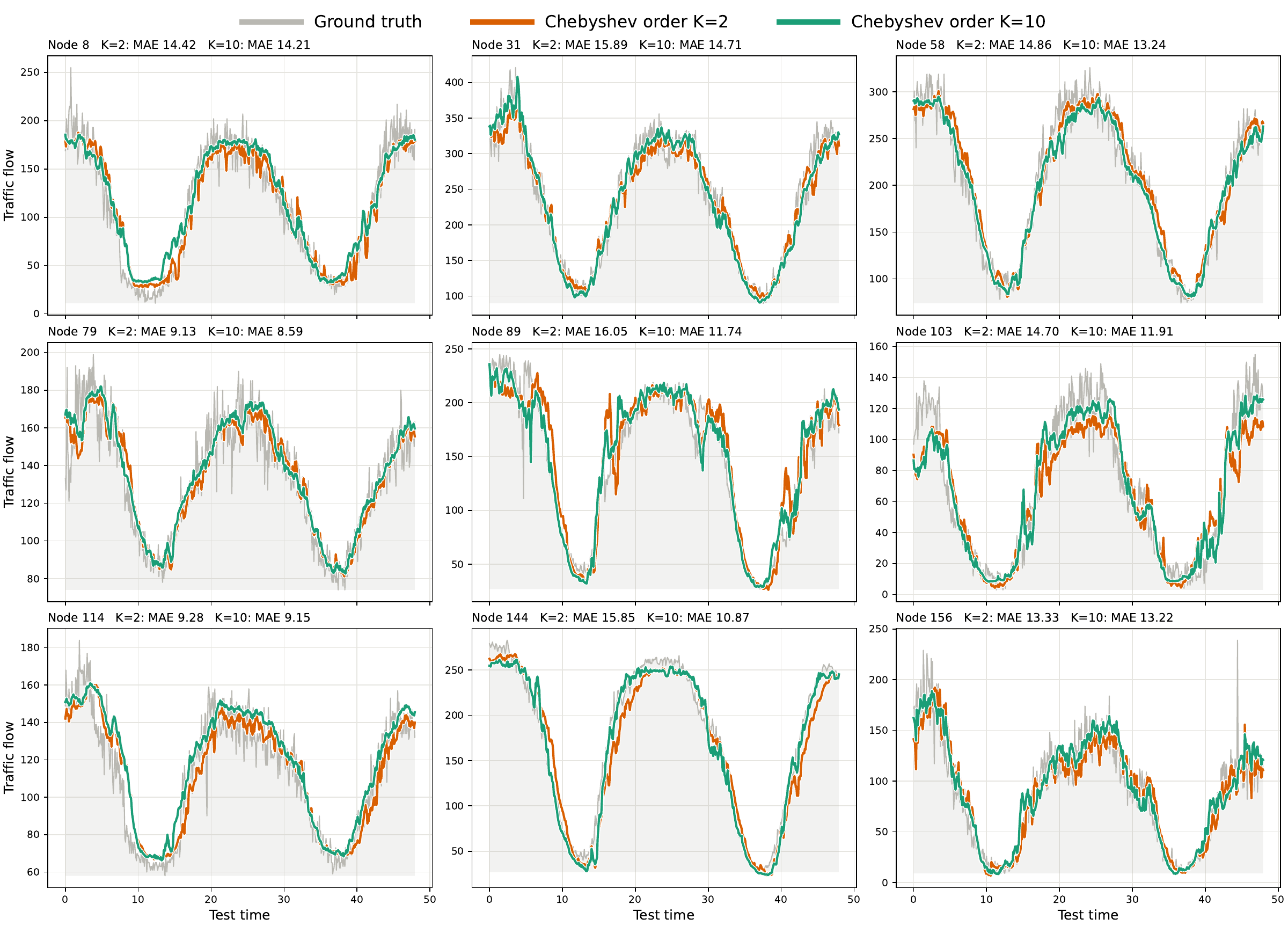}
    \caption{Ground truth versus 12-step-ahead ChronoSheaf predictions on nine PEMS08 test nodes, with Chebyshev order $K=2$ and $K=10$. Per-node MAE is shown above each panel.}
    \label{fig:cheb_pred}
\end{figure}

\subsection{Ablation Studies}
\label{app:ablation}

Table~\ref{tab:ablation} evaluates four variants of ChronoSheaf, each altering one component while keeping all other settings fixed: 
\begin{itemize}
    \item  \emph{trivial sheaf}, with all restriction maps set to the identity.
    \item $F^{\mathrm{in}} = F^{\mathrm{out}}$, tying the two transitional maps.
    \item self-loop temporal edges only.
    \item \emph{fixed $\lambda_{\max}$}, replacing the adaptive bound of Eq.~\eqref{eq:block-gershgorin} with a constant ($\lambda_{\max}=2$).
\end{itemize}

\begin{table}[t]
\centering
\caption{Ablation study on model variants across prediction horizons in MAE (lower is better).}
\label{tab:ablation}
\small
\setlength{\tabcolsep}{4pt}
\begin{tabular}{l ccccc ccc}
\toprule
& \multicolumn{5}{c}{Molene (weather)} & \multicolumn{3}{c}{Air Quality (PM2.5)} \\
\cmidrule(lr){2-6} \cmidrule(lr){7-9}
Variant & step-1 & step-2 & step-3 & step-4 & step-5 & 1--12 hr & 13--24 hr & 25--48 hr \\
\midrule
Trivial sheaf                         & 0.36 & 0.61 & 0.95 & 1.32 & 1.52 & 23.44 & 32.91 & 36.06 \\
$F^{\mathrm{in}} = F^{\mathrm{out}}$  & 0.32 & 0.55 & 0.85 & 1.07 & 1.34 & 21.71 & 30.83 & 33.72 \\
Self-loop temporal edges only         & 0.33 & 0.56 & 0.88 & 1.12 & 1.40 & 22.06 & 31.4 & 34.76 \\
Fixed $\lambda_{\max}$                & 0.32 & 0.55 & 0.84 & 1.05 & 1.29 & 22.05 & 30.88 & 33.41 \\
ChronoSheaf (full)                    & \textbf{0.31} & \textbf{0.52} & \textbf{0.80} & \textbf{1.01} & \textbf{1.24} & \textbf{21.31} & \textbf{29.90} & \textbf{32.41} \\
\bottomrule
\end{tabular}
\end{table}

\section{Complexity and Runtime Analysis}
 
Let $N$, $|\mathcal{E}|$, $P$, $d$, $r$, $K$, and $M$ denote the number of nodes, spatial edges, patches, the model and edge-stalk dimensions, the Chebyshev order, and the number of layers. ChronoSheaf never materializes $L$, which would require $\mathcal{O}(P(N+|\mathcal{E}|)d^2)$ memory. Instead, it computes $Lx$ matrix-free in $\mathcal{O}(P(Nrd + |\mathcal{E}|r))$. It projects stalks into the $r$-dimensional edge space, aggregates along spatial and temporal edges, and projects back. 

The dual spectral filter applies $L$ $K$ times. Because the gates are positive scalars, $\|V_{n,t}\|_2 = g_{n,t,1}\|B_{1,n}\|_2$ (and likewise for $F^{\mathrm{in}}$ and $F^{\mathrm{out}}$), so the spectral-radius bound $\hat\lambda_{\max}$ needs only $3N$ base-frame norms per layer and adds no backward cost. An explicit eigendecomposition would instead cost $\mathcal{O}((NPd)^3)$. The total forward cost is $\mathcal{O}\big(MP[K(Nrd + |\mathcal{E}|r) + Nd^2]\big)$, which is linear in $N$ for sparse graphs with $|\mathcal{E}| = \mathcal{O}(N)$, and also linear in $K$. 

In Table~\ref{tab:runtime}, we include the training time per epoch using different values of the Chebyshev order $K$ on different graph sizes $N$:

\begin{table}[htbp]
\centering
\caption{Training time (seconds/epoch) for different polynomial filter orders $k$ and different graph sizes $N$.}
\label{tab:runtime}
\resizebox{\textwidth}{!}{%
\begin{tabular}{l @{\qquad\textbf{(N =\ }} r @{\textbf{)}\hspace{\tabcolsep}} *{10}{c}}
\toprule
\multicolumn{2}{l}{\textbf{Polynomial Filter Order}} & $\textbf{k=1}$ & $\textbf{k=2}$ & $\textbf{k=3}$ & $\textbf{k=4}$ & $\textbf{k=5}$ & $\textbf{k=6}$ & $\textbf{k=7}$ & $\textbf{k=8}$ & $\textbf{k=9}$ & $\textbf{k=10}$ \\
\midrule
Molene       & \textbf{32}   & 1.25   & 1.82   & 2.36   & 3.04    & 3.71    & 4.19    & 4.75    & 5.43    & 5.87    & 6.29    \\
PEMS08       & \textbf{170}  & 8.021  & 11.63  & 15.22  & 19.46   & 23.86   & 26.80   & 30.57   & 34.80   & 37.70   & 40.32   \\
PEMS04       & \textbf{307}  & 14.73  & 21.28  & 28.04  & 35.61   & 43.95   & 49.07   & 56.31   & 63.72   & 69.40   & 73.89   \\
Air Quality  & \textbf{437}  & 23.60  & 34.31  & 44.65  & 57.38   & 70.04   & 79.02   & 89.77   & 102.56  & 110.71  & 118.85  \\
LargeST (SD) & \textbf{716}  & 34.89  & 50.47  & 66.38  & 84.51   & 103.96  & 116.42  & 133.18  & 151.14  & 164.23  & 175.20  \\
LargeST (CA) & \textbf{8600} & 422.17 & 611.48 & 802.35 & 1022.91 & 1257.40 & 1408.76 & 1611.02 & 1829.35 & 1986.71 & 2119.84 \\
\bottomrule
\end{tabular}%
}
\end{table}

\section{Experimental Details}
\label{app:exp}
This appendix describes the training protocol, evaluation metrics, and
hyperparameters used for all experiments in Section~\ref{sec:exp}. Unless stated
otherwise, the same protocol is applied to ChronoSheaf and to all baselines.
 
\subsection{Data Preprocessing}
 
Each dataset is split chronologically into training,
validation, and test sets. Following the protocols of the original
benchmarks, PEMS04 and PEMS08~\citep{song2020spatial} and the two LargeST
subsets~\citep{liu2023largest} use a 6:2:2 ratio. Molene and Air Quality use a
7:1:2 ratio.
Samples are extracted with a stride-one sliding window
of $T_{\text{in}}$ input steps followed by $T_{\text{out}}$ target steps.
Windows are extracted separately within each split, so that no window crosses
a split boundary.
 All input channels are z-score standardized
independently, using the mean and standard deviation computed on the training
split only. The forecast target is
standardized with its own training-split statistics. Model outputs are mapped
back to the original scale before the training loss and the evaluation
metrics are computed. The same normalization is applied to all baselines.
 
\subsection{Training}

We use four spatio-temporal blocks and set the rank of the restriction maps to 4. The model is trained using the Adam optimizer with an initial learning rate of $2.5\times10^{-3}$, $\epsilon = 10^{-4}$, and weight decay of $10^{-4}$, with the global gradient norm clipped at 5. Training windows are drawn from the training split with a stride of one, reshuffled at the start of every epoch, and grouped into mini-batches of 16. The maximum number of training epochs is set to 250. After each epoch, the model is evaluated on the validation split using the validation RMSE, computed on the original scale with missing entries masked. The learning rate is multiplied by $0.7$ whenever the validation RMSE does not improve for 10 consecutive epochs, with a lower bound of $10^{-5}$.

\subsection{Evaluation}
\label{app:exp-eval}
Let $y_{i,t}$ and $\hat{y}_{i,t}$ denote the ground truth and the prediction
for node $i$ at forecast step $t$, on the original scale. Let $\Omega$ be the
set of observed target entries. We report
$\text{MAE} = \frac{1}{|\Omega|}\sum_{(i,t)\in\Omega}
                   \left|y_{i,t}-\hat{y}_{i,t}\right|, $
    $\text{RMSE}= \sqrt{\frac{1}{|\Omega|}\sum_{(i,t)\in\Omega}
                   \left(y_{i,t}-\hat{y}_{i,t}\right)^2},$ and 
    $\text{MAPE}= \frac{100\%}{|\Omega|}\sum_{(i,t)\in\Omega}
                   \left|\frac{y_{i,t}-\hat{y}_{i,t}}{y_{i,t}}\right|$.
Unless a specific horizon is stated, metrics are averaged over all forecast steps and all test windows. 

All experiments were run on a single NVIDIA A100 GPU
(80 GB). Each configuration was trained and evaluated five times
with random seeds $\{42, 43, 44, 45, 46\}$, and all reported numbers are the
mean over the five runs.
  
\subsection{Dataset Details}
\label{app:data}
 Table~\ref{tab:app-data-stats} summarizes the six benchmarks.
\begin{table}[h]
\centering
\caption{Summary of dataset statistics.}
\label{tab:app-data-stats}
\resizebox{\linewidth}{!}{%
\begin{tabular}{llccccc}
\toprule
Dataset & Domain & Nodes & Time steps & Interval & Period & Target \\
\midrule
PEMS04       & Traffic     & 307   & 16,992 & 5 min  & 01/2018--02/2018 & Flow \\
PEMS08       & Traffic     & 170    & 17,856 & 5 min  & 07/2016--08/2016 & Flow \\
LargeST (SD) & Traffic     & 716   & 525,888 & 5 min & 01/2017--12/2021 & Flow \\
LargeST (CA) & Traffic      & 8,600 & 525,888 & 5 min & 01/2017--12/2021 & Flow \\
Molene       & Weather     & 32    & 744   & 1 h    & 01/2014          & Temperature \\
Air Quality  & Air quality & 437  & 8,760 & 1 h    & 05/2014--04/2015 & PM2.5 \\
\bottomrule
\end{tabular}}
\end{table}
 
\paragraph{Multimodal PEMS04 and PEMS08.}
\label{app:data-noaa}
For the heterogeneous-feature experiment, we augment PEMS04 and PEMS08 with
meteorological covariates from NOAA weather stations located in
District~4 and District~8, over the same periods as the traffic data (January
to February 2018 and July to August 2016, respectively). We retain four variables: dry-bulb temperature ($^\circ$C),
hourly precipitation (in), visibility (mi), and wind speed (mph). The hourly NOAA observations are aligned to the 5-minute traffic resolution
by broadcasting. NOAA data did not have any missing values. All channels are standardized independently with
training-split statistics, and traffic flow remains the sole forecast target.

\end{document}